\documentclass[sigconf]{aamas}

\usepackage{balance} 

\usepackage{algorithm}
\usepackage{algpseudocode}
\usepackage{xcolor}
\usepackage{subcaption}

\setcopyright{ifaamas}
\acmConference[AAMAS '25]{Proc.\@ of the 24th International Conference
on Autonomous Agents and Multiagent Systems (AAMAS 2025)}{May 19 -- 23, 2025}
{Detroit, Michigan, USA}{Y.~Vorobeychik, S.~Das, A.~Nowe (eds.)}
\copyrightyear{2025}
\acmYear{2025}
\acmDOI{}
\acmPrice{}
\acmISBN{}

\newcommand{\expr}{\mu} 
\newcommand{\player}{p}              
\newcommand{\arm}{a}                 
\newcommand{\matching}{m}            
\newcommand{\smatching}{m_s}         
\newcommand{\setPlayer}{\mathcal{P}} 
\newcommand{\setArms}{\mathcal{A}}   
\newcommand{\setAgents}{\mathfrak{A}} 
\newcommand{\setMatching}{\mathcal{M}} 
\newcommand{\error}{B}               
\newcommand{\algo}{\mathbf{A}}       
\newcommand{\pref}{\pi}              
\newcommand{\rank}{r}

\newtheorem{definition}{Definition}
\newtheorem{theorem}{Theorem}
\newtheorem{Lemma}{Lemma}
\newtheorem{Remark}{Remark}

\DeclareMathOperator*{\argsort}{arg\,sort}

\acmSubmissionID{174}

\title{Probably Correct Optimal Stable Matching for Two-Sided Markets Under Uncertainty}\titlenote{This paper was accepted to International Conference on Autonomous Agents and Multiagent Systems, AAMAS 2025}

\author{Andreas Athanasopoulos}
\affiliation{
  \institution{University of Neuchâtel}
  \city{Neuchatel}
  \country{Switzerland}}
\email{andreas.athanasopoulos@unine.ch}

\author{Anne-Marie George}
\affiliation{
  \institution{University of Oslo}
  \city{Oslo}
  \country{Norway}}
\email{annemage@uio.no}

\author{Christos Dimitrakakis}
\affiliation{
  \institution{University of Neuchâtel}
  \city{Neuchatel}
  \country{Switzerland}}
\email{christos.dimitrakakis@gmail.com}

\begin{abstract}
We consider a learning problem for the stable marriage model under unknown preferences for the left side of the market. We focus on the centralized case, where at each time step, an online platform matches the agents, and obtains a noisy evaluation reflecting their preferences. Our aim is to quickly identify the stable matching that is left-side optimal, rendering this a pure exploration problem with bandit feedback. We specifically aim to find Probably Correct Optimal Stable Matchings and present several bandit algorithms to do so. Our findings provide a foundational understanding of how to efficiently gather and utilize preference information to identify the optimal stable matching in two-sided markets under uncertainty. An experimental analysis on synthetic data complements theoretical results on sample complexities for the proposed methods.
\end{abstract}


\newcommand{\BibTeX}{\rm B\kern-.05em{\sc i\kern-.025em b}\kern-.08em\TeX}

\begin{document}


\pagestyle{fancy}
\fancyhead{}

\settopmatter{printacmref=false}
\setcopyright{none}
\renewcommand\footnotetextcopyrightpermission[1]{}
\pagestyle{plain}
\maketitle 


\section{Introduction}
\label{sec:intro}
We explore the problem of two-sided matching markets, where two distinct groups of agents must be matched with one another~\cite{Roth_Sotomayor_1990}. Such markets have diverse real-world applications, ranging from labor markets, like online crowd-sourcing platforms such as Amazon Mechanical Turk, to online dating services \cite{roth1984evolution, gale_Shapley_1962, Haeringer}. In these contexts, the challenge lies in designing a matching algorithm that respects the preferences of both sets of agents. 

In their seminal work, Gale and Shapley introduce the concept of stable matchings \cite{gale_Shapley_1962}. Here, agents are matched one-to-one according to ordinal preferences such that no two agents have the incentive to deviate from the proposed solution. \citeauthor{gale_Shapley_1962}'s Deferred Acceptance algorithm guarantees a stable matching that is optimal for one side of the market. This optimal stable matching always exists and is unique, based on one side's preferences. Many works have considered solutions for this and related problems of finding matchings under known preferences~\cite{manlove2013algorithmics,teo1998geometry}. However, in online settings, agents often lack certainty about their preferences.

More recently, research has focused on cases where agents' preferences are learned over time. This line of work, initiated by \citeauthor{das2005two}, frames the problem as an online decision-making challenge, likening it to a \emph{multi-agent multi-armed bandit} problem where agents must compete for potential matches. In the centralized version of this problem, a platform matches agents according to their preferences, and the agents receive noisy feedback that updates those preferences. In this setting, there is an inherent tension between exploration and exploitation: the platform must decide whether to exploit current estimates of the best stable match or continue exploring to refine agents' preferences. \citeauthor{liu2020competing} formalize this dilemma using regret measures that balance this trade-off.

In this work, we focus on the centralized version of the stable marriage model, but approach the problem from a pure exploration standpoint. Our primary goal is to \textit{efficiently} learn the optimal stable matching w.r.t. the side of the market with uncertain preferences. This perspective emerges from the observation that algorithms designed for the regret-minimization setting, which explore potential stable matches, can hinder the exploration of critical agent pairs that contribute to stability. Accordingly, we move beyond the regret-minimization framework to address the problem:
\begin{center}
\textbf{\textit{How can we efficiently identify the optimal stable matching with high probability?}}    
\end{center}

Our contributions can be summarised as follows:
\\
\textbf{(1) Action Space.} We do not limit our algorithms to stable matchings during exploration, so as to be able to learn player preferences. This is because constraining ourselves to stable matchings may never lead to discovering the optimal stable match (Section~\ref{sec:learning-setting}). \\
\textbf{(2) Solution concept.} We introduce the concept of a \textit{probably correct optimal stable matching} (Section~\ref{sec:pac}). This special case of probably approximately correct (PAC) solutions requires the output matching to be optimal with high probability. \\
\textbf{(3) Uniform exploration strategy.} We begin our analysis with an algorithm that uniformly samples all available agent pairs, similar to the Explore-Then-Commit (ETC) strategy. We demonstrate that this can produce the optimal stable matching with high probability and provide a bound on its sample complexity (Section~\ref{sec:etc}). \\
\textbf{(4) Action elimination algorithms.} Next, we explore an algorithm based on the concept of action elimination (Section~\ref{sec:elimination_algorithm}), which improves sample efficiency and reduces dependence on instance-specific parameters. Additionally, we enhance sample complexity by modifying the stopping criterion, enabling the algorithm to terminate when sufficient information is gathered (Section~\ref{sec:improved_elimination_algorithm}). \\
\textbf{(5) Adaptive sampling.} We study a strategy that adaptively samples agent pairs, improving the complexity in practice (Section~\ref{sec:adapt_sampling}). \\
\textbf{(6) Experimental evaluation.} Finally, we compare the performance of our algorithms on several synthetic instances (Section~\ref{sec:simulations}).

\section{Related Work}
\label{sec:related Work}
The \emph{stochastic} MAB problem, introduced by \cite{thompson1933likelihood}, has been extensively studied in the past few decades, becoming a fundamental tool for analyzing online decision-making problems under uncertainty.
In its classic version, a learner selects one arm in each round and receives a numeric reward drawn from an unknown distribution associated with the selected arm. The goal of maximising the learner's cumulative reward requires a balanced solution of the exploration-exploitation dilemma: finding the best arm requires exploration, while receiving good rewards requires exploitation.
Another popular setting is the pure exploration problem \cite{even2002pac,mannor2004sample, even2006action}, where the learner's objective is to identify the best action (arm to pull) regardless of the cost incurred by choosing suboptimal actions but with a minimal number of samples. This problem is usually studied in the Probably Approximately Correct (PAC) setting, where the learner aims to identify the approximately best arm with high probability. This paper focuses on the pure exploitation setting and the presented solutions liken ETC and UCB strategies. In contrast to the PAC setting, we aim at identifying the best action with high probability, not just an approximation. For a comprehensive overview of MAB, we refer to the book by \citeauthor{lattimore2020bandit}.

The stable marriage problem, is a well studied problem in two-sided markets in which agents from the two sides are matched such that no pair of agents prefers to deviate from the matching~\cite{manlove2013algorithmics,gale_Shapley_1962}, making the solution a \textit{stable matching}. In earlier work on stable matchings, preferences of agents in the market are assumed to be known.
The problem of two-sided markets with unknown preferences has recently attracted considerable attention in various settings. \citeauthor{das2005two} were the first to introduce the marriage problem with unknown preferences on both sides of the market, framing it as a MAB problem. In their work \cite{das2005two}, they empirically studied several algorithms in specific preference settings. A follow-up work \cite{liu2020competing} studied a variant of the problem where one side of the market has unknown preferences and first introduced the notion of Player Stable Optimal/Pessimal regret for the agents in the market. They studied an ETC-like algorithm, proving sublinear bounds for both notions of regret, but requiring additional knowledge of the reward differences of the arms. They also study a UCB-like approach with sublinear bounds on the Player Pessimal Stable regret. However, they identified fundamental issues with their UCB algorithm in achieving sublinear optimal stable regret for the agents. Another research direction studied the decentralized setting where the agents act independently in the market \cite{liu2021bandit, kong2022thompson, sankararaman2021dominate, basu2021beyond}. \citeauthor{kong2022thompson} introduced a Thompson sampling algorithm for the decentralized market, and they also highlight issues with their algorithm in achieving sublinear optimal/pessimal stable regret in the centralized setting. Overall, developing algorithms with sublinear optimal regret remains an open problem. 
To bridge this gap, other lines of research study the centralized problem with transferable utilities to guarantee stability \cite{cen2022regret, jagadeesan2023learning}. In this work, we instead focus on approaching the problem in a pure exploration setting, aiming to efficiently identify a correct optimal stable matching with high probability. To the best of our knowledge, the pure exploration problem has not been explored in this setting. 
\section{Problem Setting}
\label{sec:model}
In this section, we first introduce all necessary definitions of agents, their preferences and stability of matchings for the classic two-sided matching markets. We then describe the problem of learning an optimal stable matching in two-sided matching markets with unknown preferences.

\subsection{Stable matchings in two-sided markets}
\noindent\textbf{Agents:} We consider two distinct sets of agents, \textit{players} $\setPlayer = \{p_1, \cdots, p_N\}$ and \textit{arms} $\setArms = \{\arm_1, \cdots, \arm_K \}$, each having $N$ and $K$ elements, respectively, where $N \leq K$.  Let the set of all agents be $ \setAgents = \setPlayer \cup \setArms$.  We will also refer to the players as the \emph{left} and the arms as the \emph{right} side of the market.

\noindent\textbf{Preferences:} Each agent's preferences are represented as a complete list of agents from the opposite side. More specifically, the \textit{preferences} $\pref_p$ of a player $p$ are given as a permutation over the arms, i.e., $\pref_{p} = (\arm_{p_1}, \cdots, \arm_{p_K}) \in \mathbf{P}(\setArms)$, where $\mathbf{P}(\setArms)$ is the set of permutations over $\setArms$. We say player $p$ prefers arm $\pref_p(i)$ to arm $\pref_p(j)$ if $i<j$. Sometimes, we express these preference by the \textit{order relation} $\succ_{\pref_p}$ where $\pref_p(i) \succ_{\pref_p} \pref_p(j)$ if $i<j$. Here we might omit the subscript, simply writing $\succ$, if it is clear from the context. Further, we define the \textit{rank of an arm} $\arm$ w.r.t. $\pref_{p}$ as the position of $\arm$ in $p$'s preference list, i.e., $r_{p}(\arm):= | \{ \arm_j \in \setArms: \arm_j \succ_{\pref_p} \arm\} | $. Similarly, we can define preferences $\pref_{\arm}$, the corresponding order relation $\succ_{\pref_a}$, and rank function $r_a$ of an arm $\arm \in \setArms$ over the set of players $\setPlayer$.

\noindent\textbf{Matchings:} In the stable marriage model with players $\setPlayer$ and arms $\setArms$, a \textit{matching} is a set of player-arm pairs $\matching \subseteq \setPlayer \times \setArms$  that are pairwise disjoint, representing the pairs of agents that are \textit{matched}. We say an agent $i \in \setAgents$ is \textit{unmatched} under $\matching$ if there exists no pair in $\matching$ that involves $i$. By a slight abuse of notation, we define the equivalent functional representation $\matching: \setAgents \rightarrow \setAgents \cup \{\bot\}$ of a matching $\matching$, where $\matching(i) = j$ and $\matching(j) = i$ for the pair $(i, j) \in \matching$, and $\matching(i) = \bot$ if $i \in \setAgents$ is unmatched. Let \(\setMatching\)  denote the set of all possible matchings.

\noindent\textbf{Stability:} In order for a matching $\matching$ to align with the agents' preferences, Gale and Shapley proposed stability \cite{gale_Shapley_1962} as a notion of equilibrium in the market. A matching $\matching$ is \textit{stable} if there is no pair of agents who prefer to be matched with each other compared to their match under $\matching$. More formally, $\matching$ is a stable matching, if there exists no \textit{blocking pair} $(\player,\arm) \in \setPlayer\times\setArms$, i.e., no pair $(\player,\arm) \notin \matching$ such that 
(1) $ \arm \succ_{\pref_\player} \matching(\player)$ or $\player$ is unmatched, and (2) $\player\succ_{\pref_\arm} \matching(\arm)$ or $\arm$ is unmatched.

In Gale and Shapley's Deferred Acceptance (DA) algorithm~\cite{gale_Shapley_1962}, the agents on one side of the market sequentially propose to the other side, while the other side is temporarily matched with the most preferred agents so far until all agents are matched. They demonstrate that not only does a stable matching always exist for any instance of the marriage problem, but also that multiple stable matchings can exist; we denote this \textit{collection of stable matchings} by $\mathcal{S}$. The stable matching $m^{\star}_s$ produced by the algorithm is optimal for the proposing side of the market and pessimal for the side that received the proposals, in the sense that the agents of the sets are matched with their most/least preferred partner among any feasible stable matching in $\mathcal{S}$. In the remainder of this paper, we refer to the unique optimal stable matching, $m^{\star}_s$, w.r.t. the players' preferences as the \textit{optimal stable matching}.

\subsection{Two-sided markets with unknown player preferences}\label{sec:learning-setting}
In this work, we consider a setting where the players $p \in \setPlayer$ are initially uncertain about their preferences $\pref_p$. Then the players gradually learn their preferences through noisy feedback from matchings that are imposed by a centralised matching algorithm $\algo$.

More formally, the learning process is performed in $T$ rounds where in each round $t$ the algorithm $\algo$ selects a matching $\matching_t$ for the agents. Then, each player $p \in \setPlayer$ receives a stochastic reward $X^t_{p, \matching_t(p)} \in\left [0,1\right]$, distributed according to an unknown distribution $\mathbb{P}_{p, \matching_t(p)}$ with mean $\expr_{p, \matching(p)}$. We assume that a player $\player$ truly prefers arm $\arm_1$ to arm $\arm_2$ if $\expr_{p, \arm_1} > \expr_{p, \arm_2}$, while we denote the difference of their expected rewards by $\Delta_{\player,\arm_1,\arm_2} = \mid\expr_{\player,\arm_1} - \expr_{\player,\arm_2}\mid>0$.

After each round $t$, the players update their preferences $\hat{\pi}_p(t)$ based on the estimate of the expected rewards $\hat{\expr}_{p, a}(t)$ (sample mean) and report it to the algorithm. The objective of the algorithm is to efficiently identify the optimal stable matching $m^{\star}_s$ with respect to the (initially unknown) preferences $\pref_p$ with high probability. Thus, at each round $t$, the algorithm aims to select a matching $m_t \in \setMatching$ to gather relevant information. 

To summarize the process, at every round $t \in \{1, \ldots, T\}$:
\begin{enumerate}
    \item The algorithm $\algo$ selects a matching $m_t\in \setMatching$.
    \item The players $p \in \setPlayer$ receive a reward $X^t_{p, \matching_t(p)}$.
    \item The players update $\hat{\expr}_{p, m_t(p)}(t)$ and report it to $\algo$.
    \item The algorithm $\algo$ might terminate the process (according to a stopping criterion) and return a matching.
\end{enumerate}
The returned matching should be the true optimal stable matching $m^{\star}_s$ with high probability.

Note that here we do not require the action selected by $\algo$ to be a \textit{stable} matching. Since preferences of agents are uncertain, we cannot be certain which matchings are stable. Yet related work restricts $\algo$ to select only matchings that are stable with respect to confidence bounds \cite{das2005two, liu2020competing, cen2022regret}. Because we can compute confidence bounds on the estimated means $\expr_{\player}$, we could identify blocking pairs with high probability. However, the following example, which \citeauthor{liu2020competing} use to demonstrate that their UCB-like algorithm fails to achieve sublinear player-optimal stable regret \cite{liu2020competing}, shows that excluding actions of matchings with high-probability blocking pairs will in some cases not allow us to identify the optimal stable matching. This motivates our setting in which the action space of algorithm $\algo$ is the set of all matchings \(\setMatching\) disregarding stability.

\begin{example}[\citeauthor{liu2020competing} \cite{liu2020competing}]\label{ex:motivating-example}

Consider a market with three agents on each side and the following true agents' preferences.
\begin{center}
    \begin{tabular}{ll}
\toprule
\toprule
 Players                  &            Arms \\
$\pref_{\player_1} : \arm_1 \succ \arm_2 \succ \arm_3$ & $\pref_{\arm_1}: \player_2 \succ \player_3 \succ \player_1$ \\
$\pref_{\player_2} : \arm_2 \succ \arm_1 \succ \arm_3$ & $\pref_{\arm_2}: \player_1 \succ \player_2 \succ \player_3$ \\
$\pref_{\player_3} : \arm_3 \succ \arm_1 \succ \arm_2$ & $\pref_{\arm_3}: \player_3 \succ \player_1 \succ \player_2$ \\ 
\bottomrule
\bottomrule
\end{tabular}
\end{center}
Here, the only stable matchings are the player-optimal stable matching \(\matching_{s1} = \{(\player_1, \arm_1), (\player_2, \arm_2), (\player_3, \arm_3)\}\) and the arm-optimal stable matching \(\matching_{s2} = \{(\player_1, \arm_2), (\player_2, \arm_1), (\player_3, \arm_3)\}\) which is player-pessimal.

Now assume that players \(\player_1\) and \(\player_2\) are certain about their preferences and correctly report them, while player \(\player_3\)  is uncertain about the order of $\arm_1$ and $\arm_3$, but \(\player_3\)'s current estimates of the sample means $\hat{\expr}_{\player_3, \arm_1}$, $\hat{\expr}_{\player_3, \arm_2}$ and $\hat{\expr}_{\player_3, \arm_3}$ yield $\hat{\pref}_{\player_3}: \arm_1 \succ \arm_3 \succ \arm_2$. Under these preferences, the only stable matching is \(\matching_{s2}\). Further, the only matchings that could be stable w.r.t any possible preference $\hat{\pref}_{\player_3} \in \mathbf{P}(\setArms)$ of $\player_3$, are \(\matching_{s1}\) or \(\matching_{s2}\).

In this situation, if the matching algorithm $\algo$ is only permitted to select potentially stable matchings, i.e., \(\matching_{s1}\) or \(\matching_{s2}\), the player $\player_3$ will receive a sample from $\mathbb{P}_{\player_3, \arm_3}$. This serves $\player_3$ to grow more certain about the estimate $\hat{\expr}_{\player_3, \arm_3}$. However, the uncertainty about $\hat{\expr}_{\player_3, \arm_1}$ persists. Thus, if $\hat{\expr}_{\player_3, \arm_1} > \expr_{\player_3, \arm_3}$ then with high probability the updated estimated preference order $\hat{\pref}_{\player_3}$ will remain $\arm_1 \succ \arm_3 \succ \arm_2$, i.e., yielding the same situation as before. Thus, no matter the number of samples, with high probability  the returned matching is \(\matching_{s2}\) --- the player-pessimal stable matching!

In contrast, if $\algo$ is permitted to sample any matching, it could also sample a matching that matches $\player_3$ to $\arm_1$ in order to grow more certain about the estimate $\hat{\expr}_{\player_3, \arm_1}$ and eventually determining $\hat{\expr}_{\player_3, \arm_3} > \hat{\expr}_{\player_3, \arm_1}$ and correctly returning \(\matching_{s1}\).
\end{example} 

While this richer action space of all matchings (regardless of stability) allows us to circumvent the shortcomings outlined in Example 1, it also permits and even encourages the selection of matchings that are known with high probability not to be stable. Although this approach efficiently identifies a true stable matching (and is indeed necessary, as the example shows), it might not be desirable in some applications to implement such "unfair" matchings as intermediate actions. In our case, we decide to focus on the pure exploration setting, aiming only to quickly identify the optimal stable matching.

\subsection{Probably Correct Optimal Stable Matching}
\label{sec:pac}
 We introduce the notion of \textit{Probably Correct Optimal Stable Matching} (PCOS). This concept is similar to the ‘Probably Approximately Correct’ (PAC) setting \cite{even2006action}, where the goal is to find an $\epsilon$-optimal arm with high probability using as few samples as possible in the context of MAB. Note that in the context of stable matchings it is non-trivial to define what an \textit{approximation} of the optimal stable matching is. 
 Consider Example~\ref{ex:motivating-example} and assume for player $\player_3$ the gap between true mean rewards $\Delta_{\player_3,\arm_1,\arm_3} = \mid\expr_{\player_3,\arm_1} - \expr_{\player,\arm_3}\mid$ is less than $\epsilon$. Then even an $\epsilon$-approximation of the true means is not sufficient to identify the true optimal stable matching \(m_{s1}\) and instead \(m_{s2}\) is returned. In case one is tempted to consider approximations w.r.t. the sum of players' rewards over their matches we remark that here the approximation value of \(\matching_{s_2}\) towards \(\matching_{s_1}\) depends on the reward gaps $\Delta_{\player_1,\arm_1,\arm_2}$ and $\Delta_{\player_2,\arm_1,\arm_2}$ which might be arbitrarily large.
 We thus leave more extensive discussions around approximate solutions to future work and instead focus on the following solution concept.
 
\begin{definition}[Probably Correct Optimal Stable Matching]
We say that an algorithm $\algo$ is a $\delta$-PCOS algorithm with sample complexity $T$, if it outputs the optimal stable matching, $m_s^{\star}$, 
with probability at least $1-\delta$ after at most $T$ sampled matchings.
\end{definition}

\begin{Remark}
We measure the sample complexity in terms of the number of matchings performed by algorithm $\algo$.
\end{Remark}

As correctly identifying the optimal stable matching is closely related to correctly identifying the agent's preferences we introduce the following two notions. 

\begin{definition}[Completely Correct Preferences]
We say that the estimated preferences $\hat{\pref}$ of players are completely correct if they are exactly the same as the true preferences $\pref$, i.e., \(\hat{\pref}[i] = \pref[i] \; \forall \; i \in [1, \ldots, K]\).
\end{definition}
\begin{definition}[Partially Correct Preferences up to an Arm]
We say that the estimated preferences \(\hat{\pref}\) of players are partially correct up to an arm \(\arm \in \setArms\) if they are correct up to the position \(\rank_p[\arm]\) with the true preference \(\pref\), i.e., \(\hat{\pref}[i] = \pref[i] \; \forall \; i \leq \rank_p[\arm]\).
\end{definition}

We formally state the relation between the probability of the DA algorithm returning the optimal stable matching and the preference estimates being correct.
\begin{proposition}
\label{pr:probability-of-stable-matching-bound}
Let $\matching_{\hat{\pref}}$ be the output of the DA algorithm running using estimated preferences $\hat{\pref}$ for the agents $\setAgents$. The probability that $\matching_{\hat{\pref}}$ is equal to the optimal stable matching $\smatching^{\star}$ for the true preferences $\pref$ is at least as high as the probability of those estimates $\hat{\pref}$ being correct, i.e.:

\begin{equation*}
    P(\matching_{\hat{\pref}} = \smatching^{\star}) \geq P\left(\hat{\pref}_i = \pref_i \; \forall \; i \in \setAgents\right).
\end{equation*}
\end{proposition}

\begin{proof}
This follows directly from the fact that the Deferred Acceptance algorithm with correct preferences always returns the optimal stable matching $\smatching^{\star}$.
\end{proof}

\section{Naive Uniform Exploration}
\label{sec:etc}
We begin our analysis by considering an algorithm that uniformly explores every available pair of agents, similarly to the ETC algorithm in the regret minimization setting \cite{liu2020competing}. Although ETC can achieve sublinear player-optimal stable regret, this does not always imply convergence to the correct optimal stable matching. Here, we provide a uniform sampling strategy that can identify the true optimal stable matching with high probability.

Our Naive Uniform Exploration (NUE) Algorithm uniformly samples matchings such that each pair of agents is sampled for a fixed number times, depending on the minimum reward difference between the arms. After the exploration rounds, the algorithm estimates the preferences of the players using the sample mean of the rewards for the arms and commits to the matching produced by the DA algorithm using these estimated preferences. 

\begin{algorithm}
\caption{Naive Uniform Exploration (NUE)}\label{alg:etc}
\begin{algorithmic}[1]
\Require $\delta > 0$, $\setPlayer$, $\setArms$, $\Delta_{\min} = \min_{p\in \setPlayer, i, j \in \setArms} \Delta_{p, i, j} $, $\{\pref_a\}_{\arm \in \setArms}$

\State $h =  \lceil\frac{2 \ln(2KN/\delta)} {\Delta_{\min}^2}\rceil$
\For{ $t$ \ in \ $\{1, \cdots, hK\}$}
    \State $m(p_i) = \arm_j$ , $j = (t+i-2 \mod K) + 1 \; \forall i \; \in \; \{1, \dots,N\}$
    \State \textit{Sample $\matching$ 
     and update $\hat{\expr}_{p_i,m(p_i)}(t)$  $\forall i \; \in \; \{1, \dots,N\} $
    }
\EndFor
\State $\hat{\pref}_{\player} = \argsort_{\arm \in \setArms} \hat{\expr}_{\player,\arm} \; \forall \player \in \setPlayer$ 
\State $m_{\hat{\pi}} = DA(\{\hat{\pref}_p\}_{\player \in\setPlayer}, \{\pref_a\}_{\arm \in \setArms})$ 
\State \Return $m_{\hat{\pi}}$
\end{algorithmic}
\end{algorithm}

\begin{theorem}\label{th:theorem_etc}
Let $\Delta_{\min} = \min_{p\in \setPlayer, i, j \in \setArms} \Delta_{p, i, j} $. Algorithm~\ref{alg:etc} is a $\delta$-PCOS algorithm, and the number of matchings is bounded by:
\begin{equation}
    O(K\frac{\ln(KN/\delta)}{\Delta_{\min}^2}).
\end{equation}
\end{theorem}

\begin{proof}[Proof Sketch]
The sample complexity follows from the definition of the algorithm, as $O\left(Kh\right) = O(K\ln(KN/\delta)/\Delta_{\min}^2)$.

We now prove that Algorithm~\ref{alg:etc} is an $\delta$-PCOS algorithm, i.e., $\Pr(m_{\hat{\pi}} = \smatching^{\star})\leq 1-\delta$. Proposition~\ref{pr:probability-of-stable-matching-bound} implies that $\Pr(m_{\hat{\pi}} \neq \smatching^{\star}) \leq \Pr(\bigcup_{\player \in \setPlayer} \hat{\pref}_\player \neq \pref_\player) $. Using the union bound over the set of players $\setPlayer$ and arms $\setArms$ we have  $\Pr(m_{\hat{\pi}} \neq \smatching^{\star}) \leq \sum_{\player \in \setPlayer}\sum_{i=1}^{N} \Pr(\hat{\pref}_{\player}[i] \neq \pref_\player[i])$. Thus it is sufficient to show $\Pr(\hat{\pref}_{\player}[i] \neq \pref_\player[i])\leq \delta/NK.$

The event that $\hat{\pref}_{\player}[i] \neq \pref_\player[i]$ can only occur if the player $\player$ wrongly orders two consecutive arms $\arm$ and $\arm'$, i.e., $\hat{\mu}_{\player,\arm} \leq \hat{\mu}_{\player,\arm'}$ when $\mu_{\player,\arm} \geq \mu_{\player,\arm'}$. By construction of the sampled matchings (see line 3 in Algorithm~\ref{alg:etc}), every player $\player$ receives $h$ many rewards for every arm. Similar to the proof of Theorem 6 in~\cite{even2006action} and by using Hoeffding's inequality, we have that $\Pr(\hat{\expr}_{\player,\arm} \leq \hat{\expr}_{\player,\arm'}) \leq \Pr(\hat{\expr}_{\player,\arm} \leq \expr_{\player,\arm} - \Delta_{\player,\arm,\arm'} / 2) + \Pr(\hat{\expr}_{\player,\arm'} \geq \expr_{\player,\arm'} + \Delta_{\player,\arm,\arm'} / 2) 
\leq 2 \exp({- \frac{\Delta_{\min}^2}{2} h})
\leq \delta/NK$ which concludes the proof.
\end{proof}
\section{Elimination Algorithm}
\label{sec:elimination_algorithm}
In this section, we propose an elimination algorithm, similar to the one in \cite{even2006action} for MAB, which improves sample complexity and does not require prior knowledge of the reward differences \(\Delta_{\player,\arm, \arm'}\) of the agents. The key idea of our Elimination Algorithm (Algorithm~\ref{alg:elim}) is to successively eliminate player-arm pairs \((\player, \arm)\) when the position of arm \(\arm\) in player \(\player\)’s preference list can be determined with high probability. Eliminating every player-arm pair can eventually guarantee the correct estimation of the preferences of players with high probability, and thus identification of the optimal stable matching, as implied by Proposition~\ref{pr:probability-of-stable-matching-bound}.

Algorithm~\ref{alg:elim} operates in rounds, where in each round \(t\), we sample matchings such that each player \(p \in \setPlayer\) receives a reward from every non-eliminated arm from the set of available arms \(S_{\player}\), exactly ones (line 5). For this we compute a \textit{minimal matching cover}, i.e., a cardinality-minimal set of matchings $\mathfrak{M}$ that cover all remaining pairs $\{(p,a) \in \setPlayer \times \setArms \mid a \in S_\player\} \subseteq \bigcup_{m\in\mathfrak{M}} m$. We denote $\mathfrak{M}(X)$ to be a minimal matching cover on a given set of pairs $X$. 

In the case of a bipartite graph, finding a minimal matching cover can be framed as a minimum edge coloring problem, where the edges of one color correspond to a matching in the matching cover \cite{gabow1978algorithms}. According to the Kőnig-Hall Theorem, the optimal number of colors required is equal to the maximum degree of the graph --- in our case the maximal degree $\mathit{deg}_t = \Delta(G(E_t))$ of the bipartite graph $G(E_t)$ with edges $E_t$ corresponding to available pairs $\{(p,a) \in \setPlayer \times \setArms \mid a \in S_\player\}$. Thus, in every round \(t\), we sample $\mathit{deg}_t$ matchings (see also Appendix~\ref{apnd:matching_algo}).

The algorithm terminates once we eliminate all arms for every player, i.e., \(S_{\player} = \emptyset\) for all \(p \in \setPlayer\). We use the elimination rule that eliminates a \((\player, \arm)\) pair when the arm \(\arm\) has no overlapping confidence interval \(C_{p,a}\) with other available arms in \(S_{\player}\) for player \(\player\). Together with the use of anytime confidence intervals\footnote{Anytime confidence intervals ensure valid coverage probabilities at any time $t$.}, the elimination rule can guarantee the position of arm \(\arm\) in the preference list of player \(\player\) with high probability. 

\begin{algorithm}
\caption{Elimination Algorithm}\label{alg:elim}
\begin{algorithmic}[1]
\Require $\delta > 0$, $\setPlayer$, $\setArms$, $\{\pref_a\}_{\arm \in \setArms}$
\State $S_{\player} = \setArms \; \forall \player \in \setPlayer$
\State $t=1$

\State $\hat{\expr}_{\player,\arm}(t) = 0 \; \forall p \in \setPlayer, a  \in \setArms$
\While{$\mid \cup_{\player \in \setPlayer}S_{\player} \mid  \geq 1$}
\State \textit{Sample all matchings $m \in \mathfrak{M} \left(\{(p,a) \in \setPlayer \times \setArms \mid a \in S_\player\}\right)$}
\State \textit{Update $\hat{\expr}_{\player,\arm}(t)$ $\forall \player \in \setPlayer$ and $\arm  \in S_{\player}$}
\State $\error_t = \sqrt{ \frac{\ln{(4KNt^2/\delta)}}{2t}}$
\State $C_{\player,\arm} = [\hat{\expr}_{\player, \arm}(t) \pm \error_t] \; \forall \player \in \setPlayer$ and $\arm  \in S_{\player}$
\State $S_{\player} = S_{\player} \setminus \{ \arm : C_{\player, \arm} \cap C_{\player, j} = \emptyset \; \forall j \in \setArms \setminus \{\arm\} \} \; \forall \player \in \setPlayer$
\State $t = t+1$
\EndWhile
\State $\hat{\pref}_{\player} = \argsort_{\arm \in \setArms} \hat{\expr}_{\player,\arm} \; \forall \player \in \setPlayer$ 
\State \Return $DA(\{\hat{\pref}_p\}_{\player \in\setPlayer}, \{\pref_a\}_{\arm \in \setArms})$ 
\end{algorithmic}
\end{algorithm}

As a starting point for the overall sample complexity of Algorithm~\ref{alg:elim}, we first calculate how many samples $t_{\player,\arm}$ that are sufficient to eliminate an arm $\arm$ from a player $\player$'s list of available arms $S_{\player}$ as in line 9 of the algorithm. 

\begin{Lemma}
\label{Lemma:informal_t}
For a player $\player \in \setPlayer$ and arm $\arm \in \setArms$, let $\Delta_{\player,\arm} = \min_{\arm' \in \setArms \setminus \{\arm\}} \Delta_{\player, \arm, \arm'}$. With probability at least $1-\delta$, the number of samples $t_{\player,\arm}$ needed to eliminate an arm $\arm$ from $S_\player$ in line 9 of Algorithm~\ref{alg:elim}, is at most
\begin{equation}
t_{\player,\arm} = O\left(\frac{\ln{\left(KN/\delta \Delta_{\player,\arm} \right)}}{\Delta_{\player,\arm}^2}\right).
\end{equation}
\end{Lemma}
\begin{proof}[Proof Sketch]
Let $\mathcal{E}$ denote the event that, for all time steps $t$ the expected rewards $\expr_{p, \arm}$ lie in the confidence interval $CI_{p,\arm}(t)$ for every pair of player $p$ and arm $\arm$ i.e.  $\expr_{p, \arm}  - \error_t \leq \hat{\expr}_{p, \arm}(t) \leq \expr_{p, \arm}  + \error_t \; \forall \; t, \; \arm \in \setArms, \; \player \in \setPlayer$.  Using Hoeffding's inequality, we can show that the values $\error_t = \sqrt{ \frac{\ln{(4KNt^2/\delta)}}{2t}}$ in Algorithm \ref{alg:elim} correspond to bounds of
          any-time confidence intervals  $CI_{p,\arm}(t) = [ \hat{\expr}_{p, \arm}(t) - \error_t, \hat{\expr}_{p, \arm}(t) + \error_t]$. Consequently, $\mathcal{E}$ is true w.p.a. $1-\delta$.

          Now assume $\mathcal{E}$ is true. For a player $p$ consider two arms $\arm$, $\arm'$ with $\mu_{p,a} > \mu_{p,\arm'}$. Under the event $\mathcal{E}$ and for $t$ such that $\Delta_{\player,\arm, \arm'} > 4 \error_t$, the arms will have no overlapping confidence intervals as $(\hat{\mu}_{\player,\arm}-\error_t) - (\hat{\mu}_{\player,\arm'}+\error_t) \geq \mu_{\player,\arm}-\mu_{\player,\arm}  -4 \error_t  > 0$. So in order for an arm $\arm$ to have no overlapping confidence intervals with any other arm we need $\Delta_{\player,\arm} = \min_{\arm' \in \setArms \setminus \{\arm\}} \Delta_{\player, \arm, \arm'} > 4 \error_t$, which holds for 
          $t_{p,a} = O\left(\ln{\left(KN/\delta \Delta_{\player,\arm} \right)}/\Delta_{\player,\arm}^2\right)$ using a similar analysis with \cite{even2006action}.
          
          Consequently, after $t_{\player,\arm}$ samples of all arms, no confidence interval $CI_{p,\arm'}(t) \forall \arm' \in \setArms \setminus \{\arm\}$ overlaps with $CI_{p,\arm}(t)$ and $\arm$ is eliminated from $S_\player$, with probability at least $1 - \delta$.
\end{proof}
We can now state the sample complexity and show Algorithm~\ref{alg:elim} is a $\delta$-PCOS algorithm, with a detailed proof in Appendix~\ref{apnd:proof_elim}.

\begin{theorem}
\label{th:elim_algo}
Let $\mathcal{P}(S) = \{r \subseteq S: \Delta(G(S)) - \Delta(G(S \setminus r)) = 1\}$ denote the set of pairs that we have to eliminate to reduce the degree of the graph $G(S)$ with $S=\{(p,a) \in \setPlayer \times \setArms\}$. Let also $r_0 = \emptyset$ and 
\begin{equation*}
    r_i =  \arg\min_{r \in \mathcal{P}(S \setminus \bigcup_{j<i} r_j)} \max_{(p,a) \in r}  t_{p,\arm} \; \forall i=1, \dots, K
\end{equation*}
Algorithm~\ref{alg:elim} is a $\delta$-PCOS algorithm, and with probability at least $1-\delta$, the number of matching samples is bounded by:
\begin{equation}
  O\left(\sum_{s=1}^{K} \max_{(p,a) \in r_s} t_{p,\arm} \right)
\end{equation}
\end{theorem}

\begin{proof}[Proof Sketch]
    Using Hoeffding's inequality, we can show that the values $\error_t$ (in Algorithm\ref{alg:elim}) define any-time confidence intervals, i.e., the event $\mathcal{E}$, where at any time-step $t$ the expected rewards $\expr_{p, \arm}$ lie in the confidence interval $C_{p,\arm}(t)$ for every pair of player $p$ and arm $\arm$, holds with probability at least $1-\delta$.
          
    Under the event $\mathcal{E}$, once two confidence intervals are not overlapping for some arms $\arm, \arm'$, we can determine their relative order in player $p$'s preferences. This still holds true even when taking more samples of one of the arms, as the confidence intervals only shrinks. We can thus eliminate the pair $(p,\arm)$ once the order of the arm $\arm$ towards all other arms can be determined with high probability. Thus, upon termination of the algorithm, we will have determined the correct preferences $\hat{\pref}_p = \pref_p \forall \player \in\setPlayer$ and by Proposition~\ref{pr:probability-of-stable-matching-bound}, Algorithm~\ref{alg:elim} will output $m^{\star}_{s}$ with w.p.a. $1-\delta$.

    For the sample complexity, consider that at any time $t$ we sample every available pair, $\{(p,a) \in \setPlayer \times \setArms \mid a \in S_\player\}$, using the matchings from a minimal matching cover of the bipartite graph with edges corresponding to currently available pairs (see line 5). Such a minimal matching cover has $\mathit{deg}_t$ many matchings. Thus, once the maximal degree is reduced, less matchings have to be sampled. 

    We thus consider $s = 1, \dots, K$ phases where each phase corresponds to the elimination of the subset of pairs $r_s$ to reduce the degree of the graph. So under the event $\mathcal{E}$, we can use Lemma \ref{Lemma:informal_t} to define $r_s$ as the set of pairs with the lowest sample complexity i.e: $$r_i =  \arg\min_{r \in \mathcal{P}(S \setminus \bigcup_{j<i} r_j)} \max_{(p,a) \in r}  t_{p,\arm} \; \forall i=1, \dots, K$$
    with 
    $$\mathcal{P}(S) = \{r \subseteq S: \Delta(G(S)) - \Delta(G(S \setminus r)) = 1\}$$
    denote the subsets of pairs that can reduce the degree of the graph. Note that here the $r_1, \dots, r_K$ form a partition of all player-arm pairs. 
      
    In every phase $s$, we perform  $t_s - t_{s-1}$ iterations, where $t_s$ are sufficient number of steps to eliminate $r_s$ determined by the player-arm pair $(\player, \arm) \in r_s$ that takes the longest to eliminate i.e: $t_s = \max_{(p,a) \in r_s} t_{p,\arm}$ with $t_0 = 0$.
      
    Finally, in each step $t$, every matching cover consists of $K-s+1$ matchings. Thus, under the event $\mathcal{E}$ witch holds w.p.a. $1-\delta$, the total number of  matching samples is given by:
    \begin{equation*}
        \sum_{s=1}^{K} (K-s+1) (t_s-t_{s-1}) = \sum_{s=1}^{K} t_s = O\left(\sum_{s=1}^{K} \max_{(p,a) \in r_s} t_{p,\arm}\right)
    \end{equation*}
\end{proof}
\begin{Remark}
Note, that in the case where $N = K$, the set of pairs $r_s$ that we have to eliminate, corresponds to a perfect matching $m_s$ as in this case the graph of available pairs is regular at each phase $s$.
\end{Remark}
\begin{Remark}
\label{remark:Delta}
We can construct a worst-case instance, where Algorithm~\ref{alg:elim} uniformly samples all player-arm pairs. In particular, if the differences in the expected rewards are equal i.e. $\Delta_{p,a} = \Delta \; \forall p \in \setPlayer, a \in \setArms$, leading to  sample complexity:
\begin{equation*}
    O\left(\sum_{s=1}^{K} \max_{(p,a) \in r_s} t_{p,\arm} \right) = O\left(K \frac{\ln{\left(KN/\delta \Delta \right)}}{\Delta^2} \right).
\end{equation*}
\end{Remark}

\section{Improved Elimination Algorithm}
\label{sec:improved_elimination_algorithm}
In this section, we propose an improved version of the elimination algorithm, Algorithm~\ref{alg:imp_elim}, based on the observation that to identify the optimal stable matching, we only need to correctly estimate the preferences of the players up to the position of the matching partner in the optimal stable matching. We formalize this in the Lemma below.

\begin{Lemma}
\label{lemma:stopping_rule}
Let \(m_s^\star\) be the true optimal stable matching according to preferences $\pref$. The output,  \(m_{\hat{\pref}}\), of the DA algorithm using preferences $\hat{\pref}$ that are partially correct up to \(m_{\hat{\pref}}(p)\) for every player $p$, i.e., \(\hat{\pref}_{\player}[i] = \pref_{\player}[i] \; \forall \; i \leq \rank_p[m_{\hat{\pref}}(p)] \; \forall \; \player \in \setPlayer\),
is equal to the true optimal stable matching \(m_s^\star\).
\end{Lemma}

\begin{proof}
The DA algorithm sequentially executes proposals from players starting with their most preferred arm and never backtracks to a previously made proposal. Further, it halts once the player optimal stable matching has been found. Consequently it only considers the matching partner \( m_{\hat{\pi}}(p) \) and all higher ranked arms in the optimal stable matching according to $\hat{\pi}$. So if the preferences $\pref_{\player}$ are correct up to the positions of the $m_{\hat{\pi}}$ then is equal to the true optimal stable matching \(m_s^\star\).
\end{proof}
We can modify the stopping criteria of the Elimination Algorithm of the previous section according to Lemma~\ref{lemma:stopping_rule}. Algorithm~\ref{alg:imp_elim} terminates when it eliminates the arms up to the stable matching partner for every player. Specifically, after each round $t$, we calculate an estimate of the player optimal stable matching $\hat{m}_t$ from the DA algorithm using the estimated preferences $\hat{\pi}$ from the sample means. The algorithm terminates if for all players $p\in \setPlayer$, every the stable matching partner $\hat{m}_t(p)$ and all higher ranked arms have been eliminated.  As the algorithm proceeds the stable matching $\hat{m}_t$ changes until we eventually can reach a state where our termination criteria holds.

\begin{algorithm}
\caption{Improved Elimination Algorithm}
\label{alg:imp_elim}
\begin{algorithmic}[1]
\Require $\delta > 0$, $\setPlayer$, $\setArms$,  $\{\pref_a\}_{\arm \in \setArms}$
\State $t=1$, $S_\player = \setArms \; \forall \player \in \setPlayer$, $S=S_\player$
\State $\hat{\expr}_{\player,\arm}(t) = 0 \; \forall p \in \setPlayer, a  \in \setArms$
\While{$\mid S \mid  \geq 1$}
\State \textit{Sample all matchings $m \in \mathfrak{M} \left(\{(p,a) \in \setPlayer \times \setArms \mid a \in S_\player\}\right)$}
\State \textit{Update $\hat{\expr}_{\player,\arm}(t)$ $\forall \player \in \setPlayer$ and $\arm  \in S_{\player}$}
\State $\error_t = \sqrt{ \frac{\ln{4KNt^2/\delta}}{2t}}$
\State $C_{\player,\arm} = [\hat{\expr}_{\player, \arm}(t) \pm \error_t] \; \forall \arm  \in S_{\player}, \player \in \setPlayer$ 
\State $S_\player = S_\player \setminus \{ \arm : C_{\player, \arm} \cap C_{\player, j} = \emptyset \; \forall j \in \setArms \setminus \{\arm\} \}$
\State $\hat{\pref}_{\player} = \argsort_{\arm \in \setArms} \hat{\expr}_{\player,\arm} \; \forall \player \in \setPlayer$
\State $m_t = DA(\{\hat{\pref}_p\}_{\player \in\setPlayer}, \{\pref_a\}_{\arm \in \setArms})$
\State $S = \cup_{\player \in \setPlayer} \{a \in S_\player: a \succeq_{\hat{\pref}_{p}} m_t(p) \} $

\State $t = t+1$
\EndWhile
\State \Return $m_t$ 
\end{algorithmic}
\end{algorithm}

\begin{theorem}
\label{th:imp_elim}
Let $\mathcal{P}(S) = \{r \subseteq S: \Delta(G(S)) - \Delta(G(S \setminus r)) = 1\}$ denote the set of pairs that we have to eliminate to reduce the degree of the graph $G(S)$ with $S=\{(p,a) \in \setPlayer \times \setArms\}$, $r_0 = \emptyset$ and 
\begin{equation*}
    r_i =  \arg\min_{r \in \mathcal{P}(S \setminus \bigcup_{j<i} r_j)} \max_{(p,a) \in r}  t_{p,\arm} \; \forall i=1, \dots, K
\end{equation*}
Let also $t_{\max} = \max_{(p,a)\in \setPlayer\times\setArms: a \succeq_p m_s^\star(p)} t_{p,\arm}$, and $n \in \{1, \dots ,K\}$ be the index s.t. $\max_{(p,a) \in r_{n-1}} t_{p,\arm} \leq t_{\max} \leq \max_{(p,a) \in r_n} t_{p,\arm}$.

Algorithm~\ref{alg:imp_elim} is a $\delta$-PCOS algorithm, and with probability at least $1-\delta$, the number of matching samples is bounded by:
\begin{equation}
  O\left(\sum_{s=1}^{n-1} \max_{(p,a) \in r_s} t_{p,\arm} + (K-n+1) t_{\max} \right)
\end{equation}
\end{theorem}
\begin{proof}
    Algorithm~\ref{alg:imp_elim} is essentially the same as Algorithm~\ref{alg:elim} with a different stopping criterion (see lines 3 and 11), and many of arguments from the proof of Theorem~\ref{th:elim_algo} transfer. In particular, at the time of elimination of an arm $a$ from $S_\player$, the arm can be correctly ordered w.r.t. all other arms in $\player$'s preferences with high probability. Thus at the time of termination, with probability at least $1-\delta$, the condition of Lemma~\ref{lemma:stopping_rule} is satisfied and Algorithm~\ref{alg:imp_elim} outputs the correct optimal stable matching $m_s^\star$.

    For the sample complexity, we can again consider phases  $s$ where we eliminate a subset of pairs $r_s$ after at most $t_s = \max_{(p,a) \in r_s} t_{p,\arm}$ number of samples, under the event $\mathcal{E}$. However, with the stopping criterion in Algorithm~\ref{alg:imp_elim}, we might not have to eliminate every $r_1, \dots, r_K$ before stopping. In fact, by Lemma~\ref{lemma:stopping_rule} we only need to eliminate the optimal matching partner and all higher ranked arms for every player i.e. the pairs $\{(p,a) \in \setPlayer \times \setArms : a \succeq_p m_s^\star(p)\}$. For this, a maximal number of samples $t_{\max}$ from every pair are needed.
    
    Now, consider the last phase $n$ in which all arms $a \succeq_p m_s^\star(p)$ are either eliminated from $S_\player$ or sampled $t_{\max}$ many times, i.e., $t_{n-1}\leq t_{\max} \leq t_{n}$. The algorithm terminates within phase $n$, with a total number of matching: $\sum_{s=1}^{n-1} (K-s+1) (t_s - t_{s-1}) + (K-n+1)(t_{\max} - t_{n-1}) = \sum_{s=1}^{n-1} (K-s+1) t_s + (K-n+1) t_{\max}$. To conclude, under the event $\mathcal{E}$ witch holds w.p.a. $1-\delta$, the algorithm terminates with the total number of matchings bounded by:
    \begin{align*}
    O\left(\sum_{s=1}^{n-1} \max_{(p,a) \in r_s} t_{p,\arm} + (K-n+1) t_{\max} \right)
    \end{align*}
\end{proof}

\begin{Remark}
\label{remark:Delta2}
We can construct a worst-case instance, where Algorithm~\ref{alg:imp_elim} uniformly samples all player-arm pairs. In particular, if the differences in the expected rewards are equal i.e. $\Delta_{p,a} = \Delta \; \forall p \in \setPlayer, a \in \setArms$, leading to  sample complexity:
\begin{equation*}
     O\left(K \frac{\ln{\left(KN/\delta \Delta \right)}}{\Delta^2} \right).
\end{equation*}
\end{Remark}

As for Algorithm~\ref{alg:elim}, Algorithm~\ref{alg:imp_elim} will, in the worst case, have the same sample complexity as a strategy that uniformly samples pairs until they are eliminated, e.g., when all expected reward gaps are equal, i.e., $\Delta_{p,a} = \Delta \; \forall p \in \setPlayer, a \in \setArms$ (see also Remark \ref{remark:Delta}). However, in practice, we expect Algorithm~\ref{alg:imp_elim} to perform much better and terminate earlier, particularly when the reward gaps are larger for higher-ranked consecutive arm pairs. 

\section{Adaptive Sampling Algorithm}
\label{sec:adapt_sampling}
The algorithms introduced in the previous sections uniformly sample the arms for each player, until their confidence intervals are sufficiently separated. Here we propose an approach to adaptively sample player-arm pairs at every round. 

Algorithm~\ref{alg:adap_elim}, leverages insights from Lemma~\ref{lemma:stopping_rule} to dynamically define the set of arms requiring further exploration, denoted as $S_\player$. Specifically, in each round $t$, we estimate the optimal stable matching $\hat{m}_t$  from the empirical preferences $\hat{\pref}$. According to Lemma~\ref{lemma:stopping_rule}, for this matching to be correct, the arms must be accurately ranked for the players up to the partner in the optimal stable matching. Consequently, in each round $ t $ and for each player $\player$, we sample the set of arms in order to distinguish the confidence intervals of the arms up to the stable matching partner $m_t(p)$ (see line 11). Due to the adaptive selection of arms, each arm has a distinct confidence margin $ \error_{p,a} $ (line 5), which varies based on the number of times we sample a player-arm pair $ t_{p,a} $.

\begin{algorithm}
\caption{Adaptive Sampling Algorithm}
\label{alg:adap_elim}
\begin{algorithmic}[1]
\Require $\delta > 0$, $\setPlayer$, $\setArms$, $\{\pref_a\}_{\arm \in \setArms}$
\State $t_{p,a}=0$, $S_\player = \setArms \; \forall \player \in \setPlayer$
\State $\hat{\expr}_{\player,\arm}(t) = 0 \; \forall p \in \setPlayer, a  \in \setArms$
\While{$\mid \cup_{\player \in \setPlayer} S_\player \mid  \geq 1$}
\State \textit{Sample all matchings $m \in \mathfrak{M} \left(\{(p,a) \in \setPlayer \times \setArms \mid a \in S_\player\}\right)$}
\State \textit{Update $\hat{\expr}_{\player,\arm}(t)$ and $t_{p,a}$ $\forall \player \in \setPlayer$ and $\arm  \in S_{\player}$}
\State $\error_{p,a} = \sqrt{ \frac{\ln{4KNt_{p,a}^2/\delta}}{2t_{p,a}}} \forall \arm  \in \setArms, \player \in \setPlayer$
\State $C_{\player,\arm} = [\hat{\expr}_{\player, \arm} \pm \error_{p,a}] \; \forall \arm  \in \setArms, \player \in \setPlayer$ 
\State $\hat{\pref}_{\player} = \argsort_{\arm \in \setArms} \hat{\expr}_{\player,\arm} \; \forall \player \in \setPlayer$
\State $m_t = DA(\{\hat{\pref}_p\}_{\player \in\setPlayer}, \{\pref_a\}_{\arm \in \setArms})$
\State $A[\player] = \{\arm\in\setArms : \arm 
\succeq_{\hat{\pref}_{p}}m_t(p)\} \forall p \in \setPlayer$
\State $S_\player = \{\arm \in \setArms \mid \exists \arm' \in \setArms \setminus \{a\} \colon $
\Statex \hspace{1.5cm}  $C_{\player, \arm'} \cap C_{\player, \arm} \neq \emptyset \text{ and } \{a, a'\} \cap A[\player] \neq \emptyset\} \forall \player \in \setPlayer$
\State $t = t+1$
\EndWhile
\State \Return $m_t$ 
\end{algorithmic}
\end{algorithm}

\begin{theorem}
\label{th:adap_elim}
 Algorithm~\ref{alg:adap_elim} is a $\delta$-PCOS algorithm.
\end{theorem}
\begin{proof} 
First, note that, as shown for Algorithm~\ref{alg:elim} and~\ref{alg:imp_elim}, the $\error_{p,a}$ define anytime confidence intervals, i.e., the event $\mathcal{E}$ — where the expected reward lies within the confidence intervals at any time $t$ for all pairs— holds with probability at least $1 - \delta$.

So under $\mathcal{E}$, consider a round where the estimate of the stable matching $m_t$ is incorrect, witch occurs only when at least one player $\player$ has wrong preferences $ \hat{\pref}_p $ up to the stable match. This implies that at least one arm $\arm$ with $ \arm \succeq_{\hat{\pref}_p} m_t(p) $ has an overlapping confidence interval with another arm $a'$, defining the active set of arms in round $ t $, e.g., $S_\player=\{a, a'\}$. In this situation, the algorithm samples the arms in the active set $S_\player$ until their confidence intervals are sufficiently separated or until we find ourselves with a different matching $ m_{t'} $ and a new active set of arms. Since our confidence intervals only shrink, our algorithm will eventually terminate with the correct preferences up to the stable match, and thus with the correct optimal stable matching with probability at least $1-\delta$.
\end{proof}

The adaptive selection of agent pairs complicates the analysis of sample complexity, which we leave for future work. Instead, we test the algorithms in practice in the following section.
\section{Simulations}
\label{sec:simulations}

\begin{figure}[b]
    \centering
    \includegraphics[scale=0.65]{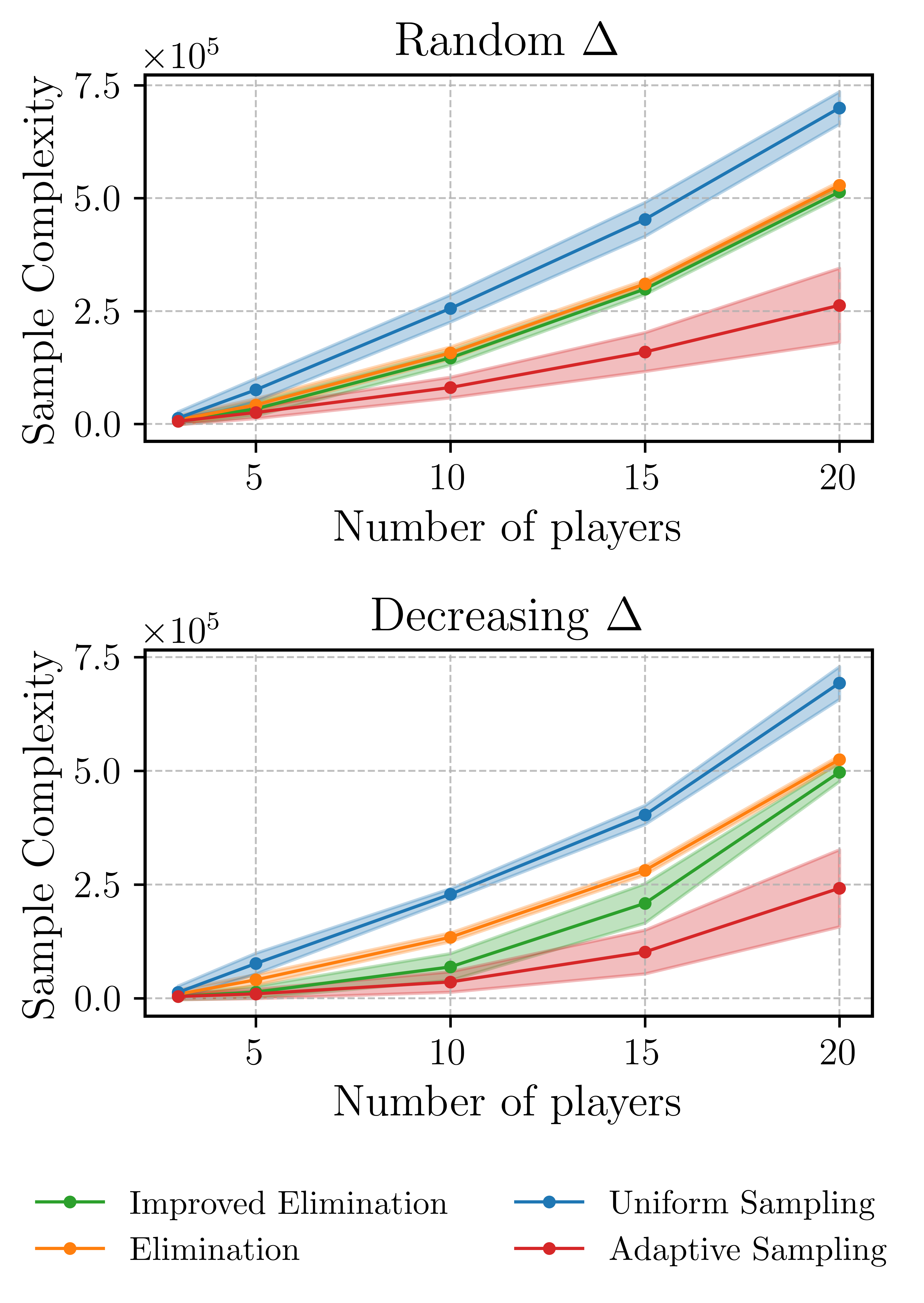}
    \caption{Sample complexity for the proposed algorithms for the two different reward settings, averaged over the runs.}
\label{fig:exp_2}
\end{figure}

\begin{figure*}[ht]
    \centering
    \includegraphics[scale=0.56]{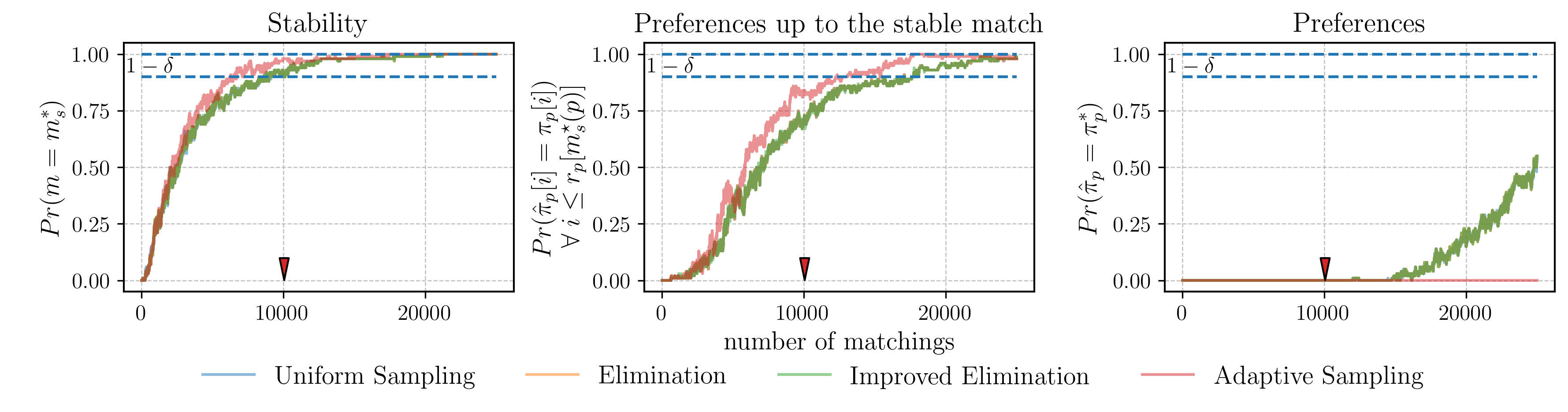}
    \caption{Any-time performance of the algorithm for the first instance with 20 agents on each side. The figure illustrates the average number of times the algorithms are able to identify (left) the optimal stable matching, (middle) the correct preferences up to the stable match for every player, and (right) the correct preferences for every player, after each matching.}
\label{fig:actual_performance}
\end{figure*}

In this section, we perform simulations on random instances to further evaluate the performance of our algorithms for fixed $\delta=0.1$. The rewards for each pair are drawn from a Bernoulli distribution, i.e.,  \(X_{p,a} \sim \text{Bern}(\mu_{p,a})\), with $\mu_{p,a} = \mathbb{E}[X_{p,a}]$. We perform different experiments by varying the number of agents with \(N = K\), while we explore two different reward settings: (1) random expected rewards, and (2) random expected rewards with decreasing gaps. For each experiment, we generate 100 random instances. The code for this work is available at \url{https://github.com/a-athanasopoulos/PACOS}.

In both reward settings, we randomly generate the preferences of the agents \(\pref\) for each instance. To create the respective expected rewards $\mu$ for the pairs, we first sample $K-1$ arbitrary reward gaps $\Delta_{\player,i,i+1}$ for each player $\player$ and $i=1, \dots, K-1$ from a Dirichlet distribution \(\text{Dir}(\alpha = 1)\), while we set $\Delta_{p,0,1}=0$. For computational reasons, we normalize these values $\Delta_{\player,i,i+1}$ to ensure that none exceed 0.05. In \textbf{Reward Setting 1}, we set the expected rewards to $\mu_{p,\pi_{p}[i]} = \sum_{j\leq K-i} \Delta_{\player,j,j+1}$, while in \textbf{Reward Setting 2}, we first sort the reward gaps in increasing order. The second setting ensures reduced sample complexity for the algorithms that employs the stopping rule according to our Lemma~\ref{lemma:stopping_rule}, as the preferences up to the stable match are easier to distinguish, i.e., $\Delta_{\player, \pref[i], \pref[i+1]} \geq \Delta_{\player, \pref[i+1], \pref[i+2]} \; \forall i=1, \dots, K-1$.

We compare the \textbf{(1) Elimination Algorithm}, \textbf{(2) Improved Elimination Algorithm}, and \textbf{(3) Adaptive Sampling Strategy}. In addition, we consider a variant of the NUE algorithm, the \textbf{(4) Uniform Sampling Strategy} that uniformly samples every pair until there are no overlapping confidence intervals, as the algorithm introduced in Section~\ref{sec:etc} requires knowledge of the minimum reward difference \(\Delta\) and has a fixed sample complexity.

\subsection{Sample Complexity}
 First note, that in every experiment the algorithms always return the correct stable matching, similar to the study on MAB \cite{pmlr-v35-jamieson14}. We further discuss the anytime  performance of the algorithms in Section \ref{subsec:anytime_perf}. In Figure~\ref{fig:exp_2}, we present the average and the standard deviation of the sample complexity over the instances, for both preference settings, respectively.

In the first setting, we observe that the elimination algorithms behave similarly. This is because the randomly generated preferences do not allow the Improved Elimination Algorithm to terminate early. However, in the second setting, the Improved Elimination Algorithm outperforms the standard Elimination strategy. Additionally, the Uniform Sampling Strategy requires significantly more samples, even in the case of 20 players, where Remark~\ref{remark:Delta} indicates similar sample complexity, as the differences of the expected rewards are the similar (\(\Delta \approx 0.05\)). This occurs because our theoretical sample complexity measures the sufficient number of samples, while in practice some arms can be eliminated earlier. Finally, the Adaptive Sampling Strategy outperforms all other algorithms in both settings, as it dynamically refines the exploration based on the agents' preferences.

\subsection{Anytime performance}
\label{subsec:anytime_perf}
Now we study the performance of the algorithms at each time step \( t \), similar to the approaches described in \cite{pmlr-v35-jamieson14, auer2002finite} for the MAB. More specifically, after each time step \( t \), we can check whether the DA algorithm using the current preferences, outputs the correct $m_s^{\star}$. We also checked if the preferences are completely correct, and correct up to the stable match, respectively. In Figure~\ref{fig:actual_performance}, show the results averaged over 100 runs for Reward Setting 1 with 20 players. The results for the remaining experiments can be found in Appendix~\ref{apnd:sim}.

First, note that the curves for the two elimination algorithms overlap, as the only difference among them lies in the stopping criterion. The same is also true for the uniform sampling strategy. This is because elimination begins after the plotted time window (after 54.000 matchings), so the algorithms uniformly sample arms until that point. On the other hand, the adaptive algorithm quickly reduces the number of matchings to be explored, as indicated by the red arrow, which marks the average time when we need to explore less agents. This allows the algorithm to efficiently identify the preferences up to the stable pair and, consequently, the optimal stable matching. In addition, the algorithm fails to accurately estimate the preferences, as it focuses only on exploring the arms crucial for stability. Finally, we can also observe that the probability of achieving the optimal stable matching is greater than the other metrics, as indicated by Proposition~\ref{pr:probability-of-stable-matching-bound} and Lemma \ref{lemma:stopping_rule}.
\section{Conclusion \& Future Work}
\label{sec:conclusion}
In this work, we consider the stable marriage problem under uncertain preferences on one side of the market. Our main objective was to develop algorithms that can efficiently identify the true optimal stable matching with high probability. To this end, we proposed the novel concept of a Probably Correct Optimal Stable Matching and present several algorithms accompanied by theoretical analyses of their correctness and sample complexity. Finally, we support our theoretical results with an empirical evaluation of the performance of our algorithms in practice.

There are several interesting directions for future research. To begin, a promising area is the hardness analysis, particularly establishing a lower bound for the pure exploration problem. Another straightforward extension is to investigate how these algorithms perform when both sides of the market are uncertain about their preferences (discussed in Appendix~\ref{apnd:both_sides}). Additionally, analyzing the expected sample complexity \cite{mannor2004sample} would provide insights into the algorithm's performance in practice. In particular, the sample complexity of the Adaptive Sampling Strategy remains an open question, while experiments showcase its superiority over the other strategies. Another idea is applying the principles of the Elimination Algorithm to the regret minimization setting \cite{liu2020competing}. Finally, exploring alternative solution concepts, such as popular matchings \cite{huang2013popular} and median stable matchings \cite{teo1998geometry}, or study alternative models for two-sided markets \cite{Roth_Sotomayor_1990}, can potentially broaden the applicability of our approach.

\newpage
\bibliographystyle{ACM-Reference-Format} 
\bibliography{sample}

\appendix
\section{Uncertainty for both sides of the market.}
\label{apnd:both_sides}
In the main text, we discuss the case where only one side of the market is uncertain about its preferences. An interesting research direction is to consider the scenario where both sides of the market have uncertain preferences.

Here, we discuss how our results can potentially generalize in this setting. To begin with, our Proposition \ref{pr:probability-of-stable-matching-bound} also holds in this case, as we consider every agent in the market. The only difference is that we need to adjust the bounds in the proofs to account for the event that every agent has correctly estimated their preferences. Therefore, there must be additional factors based on the number of agents inside the logarithm.

Regarding the algorithms, the analysis of the uniform sampling strategy is straightforward, as one can obtain enough samples based on the minimum reward difference \( \Delta \) for both sides. For the case of the elimination algorithm, adjustments can also be made if a player-arm pair is eliminated when both agents are able to separate the confidence intervals w.r.t the agents of the other side of the market. Additionally, the stopping rule used in the Improved Elimination Algorithm, as outlined in Lemma \ref{lemma:stopping_rule}, can also be modified to consider preferences that are correct up to any stable match for both sides of the market. This can also be used to adjust the adaptive sampling strategy.

\section{Proof of Theorem \ref{th:theorem_etc}}
\label{apnd:proof_etc}

\begin{proof}
We define the event \( C_p \) as the event in which one player \( p \in \setPlayer \) correctly estimates their preferences, and the event \( C \) the event in which all players correctly estimate their preferences. Additionally, we let the event \( \mathcal{E} \) represent the case that the algorithm correctly outputs the optimal stable matching.

\noindent\emph{\textbf{Step 1 - bounding $P\left(\mathcal{E} \right)$}}\\
\noindent From Proposition \ref{pr:probability-of-stable-matching-bound} we know that:
\begin{equation*}
    P\left(\mathcal{E} \right)\geq P\left( C \right) = 1 - P\left( C^c \right) = 1 - P( \bigcup_{p\in\setPlayer} C^c_p),
\end{equation*}
where $C^c_p$ denote the opposite event of $C_p$.
Using the union bound we have that:
\begin{equation*}
    P\left(\mathcal{E}\right) \geq 1 - \sum_{p\in\setPlayer} P(C^c_p).
\end{equation*}
In order for our algorithm to be \(\delta\)-PACOS, i.e., to bound \(P\left(\mathcal{E}\right) \geq 1 - \delta\), we need:
\begin{equation*}
    P\left(C^c_p\right) \leq \delta/N \quad \forall p \in \setPlayer.
\end{equation*}

\noindent\emph{\textbf{Step 2 - bounding $P\left(C^c_p\right)$}} \\
The event $C_p$ for a player $p$ to correctly estimate their preference is equal to the intersection of the events that all consecutive arms (according to the preference) are correctly ordered, i.e., $Pr(C_p) = Pr(\bigcap_{i=1}^{K-1} C_{p, i, i+1})$. Here $C_{p, i, i+1}$ denotes the event that the $i$-th ranked agent $\arm$ and $(i+1)$-th ranked agent $a'$ in $\hat{\pref}_p$ are correctly ordered, i.e., $\expr_{p,\arm} > \expr_{p,a'}$. The opposite event $C^c_p$ can be bounded using the union bound as:
\begin{align*}
    &P\left(C^c_p\right) =  P\left(\bigcup_{i=1}^{K-1} C^c_{p, i, i+1} \right)  \\ 
    &\leq \sum_{i=1}^{K-1} P\left(C^c_{p, i, i+1} \right) \leq K \max_i P\left(C^c_{p, i, i+1} \right).
\end{align*}
\noindent  In order for $P\left(C^c\right) \leq \delta$ to hold we next aim to establish the bound $\max_{p,i} P\left(C^c_{p, i, i+1} \right) \leq \delta/NK$.

\noindent\emph{\textbf{Step 3 - bounding $P\left(C^c_{p, i, i+1}\right)$}} \\
\noindent Consider two consecutive arms, \(\arm_i\) and \(\arm_j\), with \(\Delta_{\player,i,j} = \expr_{\player,i} - \expr_{\player,j} > 0\) for an arbitrary player \(\player\). The event \(C^c_{\player,i,j}\) that the algorithm incorrectly orders the arms \(\arm_i\) and \(\arm_j\) after \(h\) samples from each arm, occurs when \(\hat{\expr}_{\player,i}(h) \leq \hat{\expr}_{\player, j}(h)\). Using the Hoeffding inequality, we can upper bound the probability of the event \(C^c_{\player,i,j}\) as follows, omitting the indexing for player \(p\) for simplicity:

\begin{align*}
&\Pr(C^c_{\player,i,j}) = \Pr(\hat{\expr}_{i}(h) \leq \hat{\expr}_{j}(h)) \\
& \leq \Pr\left(\hat{\expr}_{i}(h) < \expr_{i} - \frac{\Delta_{i,j}}{2}\right) + \Pr\left(\hat{\expr}_{j}(h) > \expr_{j} + \frac{\Delta_{i,j}}{2}\right) \\
& \leq 2 \exp\left(- \frac{\Delta_{i,j}^2}{2} h\right)
\end{align*}


\noindent\emph{\textbf{Step 4 - sample complexity}} \\
In order to bound $\max_{p,i} P\left(C^c_{p, i, i+1} \right) \leq \delta/NK$ we need:
\begin{equation*}
 2 \exp{\frac{-h \Delta_{min}^2}{2}} \leq \delta/NK \Rightarrow  \\
 h \geq \frac{2 \ln(2KN/\delta)}{\Delta_{\min}^2},
\end{equation*}
where $\Delta_{\min} = \min_{p, i, j} \Delta_{p, i, j} $

To conclude, w.p. at leat $1 - \delta$, the sample complexity of the algorithm is
\begin{equation*}
O(hK) = O(\frac{K \ln(KN/\delta)}{\Delta_{\min}^2} ).
\end{equation*}
\end{proof}
\section{The minimum matching cover}
\label{apnd:matching_algo}
This section explains how we match agents in each round of our process in order to efficiently sample all pairs $(p,a)$ (line 5 of Algorithm \ref{alg:elim}). Specifically, in each round $t$, we sample every non-eliminated pair, i.e., $(p, a) \in \setPlayer \times S_p$,  exactly once using the minimum number of matchings.

Now, consider the bipartite graph \( G_t = \{V, E_t\} \), where the vertices \( V = \{\setPlayer, \setArms\} \) represent the players and arms, and the edges $E_t$ connect the non-eliminated pairs $(p, a) \in \setPlayer \times S_p$ at round $t$. The objective is to cover all available pairs using the minimum number of matchings, i.e., a cardinality-minimal set of matchings $\mathfrak{M}$ that covers all remaining pairs.

In bipartite graphs, this problem can be formulated as a \emph{minimum edge coloring problem}, where matchings correspond to edges of the same color \cite{gabow1978algorithms}. According to the Kőnig-Hall Theorem, for a bipartite graph \( G \), the minimum number of colors required for edge coloring is equal to the maximum degree \( \Delta(G) \) of its vertices. Thus, we need \( \Delta(G_t) \) matchings to cover all edges in the graph at every round $t$.

Our method can be implemented using any edge coloring algorithm. In the experiments section, we use a recoloring edge coloring algorithm that recolors edges using augmenting paths when no available colors are left, as described in \cite{gabow1978algorithms}.

\section{Proof of Theorem \ref{th:elim_algo}}
\label{apnd:proof_elim}

\begin{proof} \textit{ } \\
\noindent \emph{\textbf{step 1 - Anytime confidence intervals}} \\
We start our proof by noticing that the values $\error_t = \sqrt{ \frac{\ln{(4KNt^2/\delta)}}{2t}}$ in Algorithm \ref{alg:elim} correspond to bounds of
any-time confidence intervals $CI_{p,\arm}(t) = [ \hat{\expr}_{p, \arm}(t) - \error_t, \hat{\expr}_{p, \arm}(t) + \error_t]$. More specifically, with probability at least $1-\delta$, at any time-step $t$, the expected rewards $\expr_{p, \arm}$ lie in the confidence interval $CI_{p,\arm}(t)$ for every pair of player $p$ and arm $\arm$. Let us denote this event by:
\begin{align}
    \mathcal{E} = \bigcap_{p \in \setPlayer} \bigcap_{a \in \setArms} \bigcap_{t = 1}^{\inf} \{ \mid \hat{\expr}_{p,\arm}(t) - \expr_{p,\arm} \mid \leq \error_t \}
\end{align}
To prove that $\Pr(\mathcal{E})\geq 1-\delta$, we first consider the individual events $\mathcal{E}_{p, \arm}(t)$ that the true mean $\expr_{p, \arm}$ lies in the confidence interval $CI_{p,\arm}(t)$ for some player $p$, arm $\arm$, and time $t$. We can bound the probability of the complement event $\mathcal{E}_{p, \arm}^c(t)$, i.e., $\Pr[\mid \hat{\expr}_{p,\arm}(t) - \expr_{p,\arm} \mid > \error_t]$ using the Hoeffding inequality and $\error_t = \sqrt{ \frac{\ln{(4KNt^2/\delta)}}{2t}}$ as:
\begin{align}
    \Pr[\mathcal{E}^c_{p, \arm}(t)] = \Pr[\mid \hat{\expr_i}(t) - \expr_i \mid > \error_t]\leq 2\exp(-2{\error_t}^2t) \leq \frac{\delta}{2NK t^2}
\end{align}
Using the union bound over all rounds,
\begin{align}
    \Pr[\bigcup_{t=1}^{\inf} \mathcal{E}_{p, \arm}^c(t) ] \leq \sum_{t=1}^{\inf} \Pr[\mathcal{E}_{p, \arm}^c(t) ] \leq \frac{\delta}{NK}
\end{align}
Using the union bound over all possible players and arms, we can bound the probability that the mean estimate for some arm is outside the confidence intervals, i.e.: 
\begin{align}
    \Pr[\mathcal{E}^c] \leq \sum_{p \in \setPlayer} \sum_{a \in \setArms} \Pr[\bigcup_{t=1}^{\inf} \mathcal{E}_{p, \arm}^c(t)] \leq \delta
\end{align}
So at any time \( t \), for every player, all arms have estimated rewards within their respective confidence intervals with probability at least \( 1 - \delta \):
\begin{align}
\label{eq:event_anytime_conf_int}
    \Pr[\mathcal{E}] &\geq 1 - \delta
\end{align}

\noindent \emph{\textbf{step 2 - Correctness argument}} \\
Now we will prove that our algorithm outputs the correct optimal stable matching $m^{\star}_s$ with probability at least \( 1 - \delta \). To prove this, we first demonstrate that as long as the event \(\mathcal{E}\) holds, our algorithm returns the correct ranking for every player. This is sufficient for obtaining the correct match. Then according to Proposition \ref{pr:probability-of-stable-matching-bound}, we can compute the correct stable optimal matching using the DA algorithm based on these preferences.

In the following, we drop the $p$ subscript to simplify notation. From \ref{eq:event_anytime_conf_int}) we know that event \(\mathcal{E}\) holds with high probability. In that case, the following hold:
\begin{enumerate}
    \item First, our algorithm will eventually terminate, as \(\error_t\) approaches zero as \(t\) increases. This ensures that there will be no overlapping arms for every player.
    \item Second, consider two arbitrary arms \(\arm_i\) and \(\arm_j\) with \(\expr_i \geq \expr_j\). If the arms have no overlapping confidence intervals, i.e., \(\hat{\expr}_i - \hat{\expr}_j \geq 2\error_t\), we can determine their relative ranking, i.e., \(\expr_i > \expr_j\), as follows:
    \begin{align}
        \expr_i - \expr_j &\geq \hat{\expr}_i(t) - \error_t - (\hat{\expr}_j(t) + \error_t) \\
        &\geq \hat{\expr}_i(t) - \hat{\expr}_j(t) - 2\error_t \\
        &\geq 0.
    \end{align}
    Thus, we conclude that \(\expr_i \geq \expr_j\) as required.   
    
    \item Third, if an arm is eliminated at time \(t\), it maintains its relative order according to the mean estimates when we continue to sample the remaining arms, as the confidence intervals of the other arms only shrink as \(t\) increases. 
    \\
    To illustrate this, consider two arms \(\arm_{i}\) and \(\arm_j\) with \(\Delta_{i,j} = \expr_{\arm_{i}} - \expr_{\arm_j} > 0\), without loss of generality. Let also \(\arm_i\) be eliminated at some round \(t\) while we continue to sample \(\arm_j\). Since arm \(\arm_i\) is eliminated, we know that \(\hat{\Delta}_{i,j} = \hat{\expr}_i(t) - \hat{\expr}_j(t) \geq 2\error_t\). 
    \\
    In some \(t' > t\), where we continue to sample \(\arm_j\),  we have a new estimate of the expected rewards, \(\hat{\expr}_j(t')\). Now consider the worst-case scenario where at round \(t\) we underestimate the expected rewards and at round \(t'\) we overestimate them, leading to:
    \begin{align}
      \hat{\expr}_j(t') - \hat{\expr}_j(t)  &\leq \expr + a_{t'} - (\expr - a_{t}) \leq 2\error_t \Rightarrow \\
      &\hat{\expr}_j(t') \leq \hat{\expr}_j(t) + 2\error_t.
    \end{align}
    Therefore, the relative position of the arm remains the same as:
    \begin{align}
        \hat{\expr}_{i}(t) - \hat{\expr}_{j}(t') &\geq \hat{\expr}_{i}(t) - \hat{\expr}_{j}(t) - 2\error_t \\ 
        &= \hat{\Delta}_{i,j} - 2\error_t \geq 0 \Rightarrow \\
        & \hat{\expr}_{\arm_{i}}(t) \geq \hat{\expr}_{\arm_j}(t').
    \end{align}
    Thus, the eliminated arm maintains its relative position according to the mean estimates as long as the average rewards are between the confidence intervals.
    
    \item Fourth, when an arm is eliminated, we know its relative position compared to all other arms with high probability, as it has no overlapping confidence intervals with any other arm.
    
    \item Finally, when all arms are eliminated for a player, we know the relative position between every arm, and thus we have the correct preferences for the player with high probability. 
\end{enumerate}

To conclude, using Proposition \ref{pr:probability-of-stable-matching-bound} our algorithm outputs the correct optimal stable matching with probability at least \(1 - \delta\).

\noindent \emph{\textbf{step 3 - Sample Complexity}} \\
We now begin the analysis of the sample complexity by first presenting a Lemma that bounds the sufficient number of samples required to eliminate an arbitrary arm for a player. We then derive the overall sample complexity by accounting for the elimination process, as outlined in the main text.\\

\noindent \emph{\textbf{step 3.1 - Sample Complexity to eliminate an arm $t_{p,\arm}$}} \\
Now we present a detailed proof for Lemma \ref{Lemma:informal_t2}. \\

\setcounter{Lemma}{0}
\begin{Lemma}
\label{Lemma:informal_t2}
For a player $\player \in \setPlayer$ and arm $\arm \in \setArms$, let $\Delta_{\player,\arm} = \min_{\arm' \in \setArms \setminus \{\arm\}} \Delta_{\player, \arm, \arm'}$. With probability at least $1-\delta$, we need at most $t_{p,\arm}$ samples to eliminate an arm $\arm$ from $S_\player$ in line 9 of Algorithm~\ref{alg:elim}:
\begin{equation}
t_{\player,\arm} = O\left(\frac{\ln{\left(KN/\delta \Delta_{\player,\arm} \right)}}{\Delta_{\player,\arm}^2}\right).
\end{equation}
\end{Lemma}
\setcounter{Lemma}{2}

\begin{proof}
\noindent First, consider some fixed player $p$ and drop the subscript $p$ for simplicity. Let us take two arbitrary arms $\arm_i$ and $\arm_{j}$ with $\expr_i \geq \expr_j$. In order for the arms to have non-overlapping confidence intervals, we want:
\begin{equation}
\label{eq:no_overlap}
    \hat{\Delta}_{i,j} = \hat{\expr}_i(t) -\hat{\expr}_j(t)  > 2 \error_t
\end{equation}
If the event $\mathcal{E}$ holds for all the arms and time steps,  the mean estimates of the arms are within the confidence intervals. \\
So for an arbitrary arm $\arm_k$ we have:
\begin{equation}
    \expr_k + \error_t \geq  \hat{\expr}_k(t) \geq \expr_k - \error_t 
\end{equation}
and thus:
\begin{equation}
    \hat{\expr}_i(t) \geq \expr_i - \error_t  \; \; \; \textit{  and  } \; \; \;  
 - \hat{\expr}_j(t) \geq -(\expr_j + \error_t)  
\end{equation}
So the event of equation \ref{eq:no_overlap} is guaranteed to occur if $\Delta_{i,j} > 4\error_t$ as it implies that:
\begin{align}
    \hat{\Delta}_{i,j} = \hat{\expr}_i(t) - \hat{\expr}_j(t) &\geq  (\expr_i -\error_t) - (\expr_j + \error_t) = \Delta_{i,j} - 2\error_t \Rightarrow \\
    & \hat{\Delta}_{i,j} > 2\error_t
\end{align}
By solving for the minimum value of \( t \) for which \( \Delta_{i,j} > 4\error_t \), we can upper bound the number of samples sufficient for the arms to have non-overlapping confidence intervals. Similarly to the analysis in \cite{even2006action} we have:
\begin{equation}
t_{\arm_i,\arm_j} = O\left(\frac{\ln{(NK/\delta\Delta_{i,j})}}{\Delta_{i,j}^2}\right)
\end{equation}

Now we can apply this result to each player $p$. In order for an arm $\arm$ to be eliminated it must have non-overlapping CI with every other arms of the player $p$. So, under the event $\mathcal{E}$ witch holds w.p.a. $1-\delta$, the sufficient number of samples $t_{p,\arm}$ are:
\begin{equation}
t_{p,\arm} =O(\frac{\ln{(NK/\delta\Delta_{p,\arm})}}{\Delta_{p,\arm}^2})  
\end{equation}
where $\Delta_{\player,\arm} = \min_{\arm' \in \setArms \setminus \{\arm\}} \Delta_{\player, \arm, \arm'}$, as it has to resolve the uncertainty of the closest arm. \\

\end{proof}

\noindent \emph{\textbf{step 3.2 - Overall sample complexity}} \\
Here we measure sample complexity in terms of the number of matchings that are sufficient for outputting the correct optimal stable match with high probability.

In each round \( t \), our algorithm samples every non-eliminated player-arm pair exactly once. The number of matchings in each round $t$ corresponds to the degree \( \Delta(G_t) \) of the underlying bipartite graph \( G_t = \{V, E_t\} \), with vertices \( V = \{\setPlayer, \setArms\} \) and the edges $E_t$ connect the non-eliminated pairs $\{(p,a) \in \setPlayer \times \setArms \mid a \in S_\player\}$. To reduce the number of matchings in each round, we have to reduce the degree of the graph.

We thus consider $s = 1, \dots, K$ phases where each phase corresponds to the elimination of the subset of pairs $r_s$ to reduce the degree of the graph. So under the event $\mathcal{E}$, we can use Lemma \ref{Lemma:informal_t2} to define $r_s$ as the set of pairs with the lowest sample complexity i.e.: $$r_i =  \arg\min_{r \in \mathcal{P}(S \setminus \bigcup_{j<i} r_j)} \max_{(p,a) \in r}  t_{p,\arm} \; \forall i=1, \dots, K$$
with 
$$\mathcal{P}(S) = \{r \subseteq S: \Delta(G(S)) - \Delta(G(S \setminus r)) = 1\}$$
denote the subsets of pairs that can reduce the degree of the graph. Note that here the $r_1, \dots, r_K$ form a partition of all player-arm pairs. 
      
In every phase $s$, we perform  $t_s - t_{s-1}$ iterations, where $t_s$ are sufficient number of steps to eliminate $r_s$ determined by the player-arm pair $(\player, \arm) \in r_s$ that takes the longest to eliminate i.e.: $t_s = \max_{(p,a) \in r_s} t_{p,\arm}$ with $t_0 = 0$.

Finally, in each step $t$, every matching cover consists of $K-s+1$ matchings. Thus, under the event $\mathcal{E}$ witch holds w.p.a. $1-\delta$, the total number of  matching samples is given by:
    \begin{equation*}
        \sum_{s=1}^{K} (K-s+1) (t_s-t_{s-1}) = \sum_{s=1}^{K} t_s = O\left(\sum_{s=1}^{K} \max_{(p,a) \in r_s} t_{p,\arm}\right)
    \end{equation*}
    
To conclude our Algorithm~\ref{alg:elim} outputs the correct stable matching with probability at least $1 - \delta$ after at most:
\begin{equation*}
       O\left(
            \sum_{s=0}^{K} \max_{(p,a) \in m_s} t_{p,\arm}
        \right) = O\left(
            \sum_{s=0}^{K} \max_{(p,a) \in m_s} \frac{\ln{\left(KN/\delta \Delta_{\player,\arm} \right)}}{\Delta_{\player,\arm}^2}
        \right) 
\end{equation*}
matching samples.   
\end{proof}



\section{simulations}
\label{apnd:sim}
In Figures \ref{fig:actual_performance4} and \ref{fig:actual_performance5}, we present the results on the anytime performance of the algorithm, similar to the subsection \ref{subsec:anytime_perf}, for different numbers of players and both settings, respectively.

\begin{figure*}[h]
    \centering
    \includegraphics[scale=0.45]{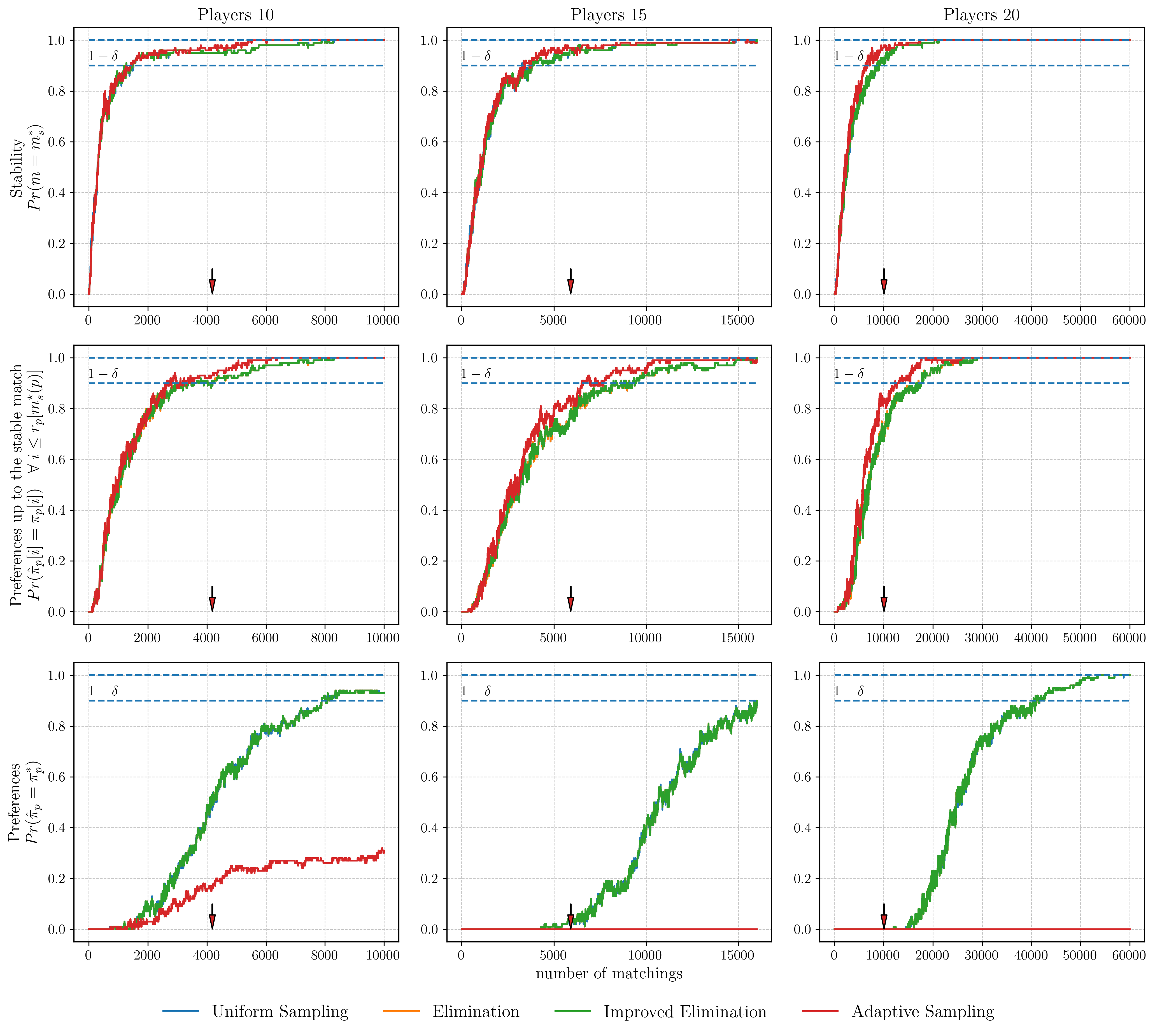}
    \caption{Anytime performance of the algorithm for the first instance setting with varying numbers of agents on each side. The figure illustrates the average number of times the algorithm identifies (left) the optimal stable matching, (middle) the correct preferences up to the stable match for each player, and (right) the correct preferences for every player after each matching.}

\label{fig:actual_performance4}
\end{figure*}

\begin{figure*}[h]
    \centering
    \includegraphics[scale=0.45]{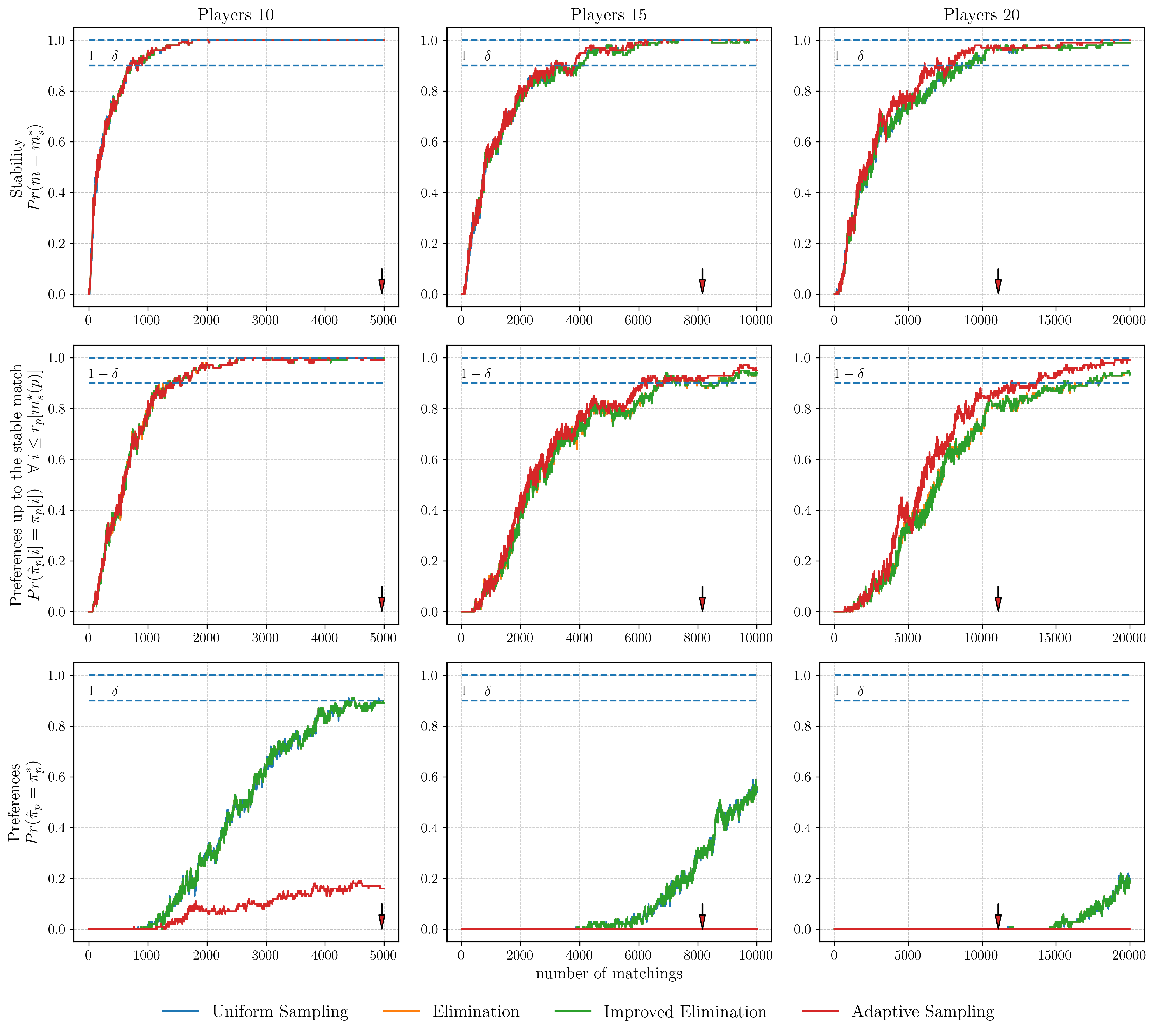}
    \caption{Anytime performance of the algorithm for the second instance setting with varying numbers of agents on each side. The figure illustrates the average number of times the algorithm identifies (left) the optimal stable matching, (middle) the correct preferences up to the stable match for each player, and (right) the correct preferences for every player after each matching.}
    \label{fig:actual_performance5}
\end{figure*}

\end{document}